\documentclass[11pt,twoside]{article}
\usepackage[margin=1in]{geometry}

\usepackage[utf8]{inputenc} 
\usepackage[T1]{fontenc}    
\usepackage{url,hyperref,xspace}
\usepackage[usenames]{color}
\usepackage{times}
\usepackage{verbatim,makeidx}
\usepackage{algorithm}
\usepackage{algorithmic}
\usepackage{amsmath}
\usepackage{amssymb}
\usepackage{amsthm}
\usepackage{dsfont}
\usepackage{color}
\usepackage{colortbl}
\usepackage{float}
\usepackage{cite}
\usepackage{pbox}
\usepackage{bbm}
\usepackage{authblk}
\usepackage{graphicx}
\usepackage{caption}
\usepackage{subcaption}

\allowdisplaybreaks

\theoremstyle{plain}
\newtheorem{theorem}{Theorem}
\newtheorem{proposition}{Proposition}
\newtheorem{lemma}{Lemma}
\newtheorem{claim}{Claim}
\newtheorem{remark}{Remark}

\theoremstyle{definition}
\newtheorem{definition}{Definition}

\theoremstyle{remark}

\newcommand{\beq}{\begin{eqnarray}}
\newcommand{\eeq}{\end{eqnarray}}

\newfont{\bbb}{msbm10 scaled 500}

\newfont{\bb}{msbm10 scaled 1100}

\newcommand{\RR}{\mbox{\bb R}}

\newcommand{\ev}{{\bf e}}

\newcommand{\Ec}{{\cal E}}

\newcommand{\Gc}{{\cal G}}

\newcommand{\Lc}{{\cal L}}
\newcommand{\Mc}{{\cal M}}
\newcommand{\Nc}{{\cal N}}

\newcommand{\Rc}{{\cal R}}
\newcommand{\Sc}{{\cal S}}

\newcommand\nc\newcommand
\nc\bfa{{\boldsymbol a}}\nc\bfA{{\boldsymbol A}}\nc\cA{{\mathcal A}}
\nc\bfb{{\boldsymbol b}}\nc\bfB{{\boldsymbol B}}\nc\cB{{\mathcal B}}
\nc\bfc{{\boldsymbol c}}\nc\bfC{{\boldsymbol C}}\nc\cC{{\mathcal C}}
\nc\sC{{\mathscr C}}
\nc\bfd{{\boldsymbol d}}\nc\bfD{{\boldsymbol D}}\nc\cD{{\mathcal D}}
\nc\bfe{{\boldsymbol e}}\nc\bfE{{\boldsymbol E}}\nc\cE{{\mathcal E}}
\nc\bff{{\boldsymbol f}}\nc\bfF{{\boldsymbol F}}\nc\cF{{\mathcal F}}
\nc\bfg{{\boldsymbol g}}\nc\bfG{{\boldsymbol G}}\nc\cG{{\mathcal G}}
\nc\bfh{{\boldsymbol h}}\nc\bfH{{\boldsymbol H}}\nc\cH{{\mathcal H}}
\nc\bfi{{\boldsymbol i}}\nc\bfI{{\boldsymbol I}}\nc\cI{{\mathcal I}}
\nc\bfj{{\boldsymbol j}}\nc\bfJ{{\boldsymbol J}}\nc\cJ{{\mathcal J}}
\nc\bfk{{\boldsymbol k}}\nc\bfK{{\boldsymbol K}}\nc\cK{{\mathcal K}}
\nc\bfl{{\boldsymbol l}}\nc\bfL{{\boldsymbol L}}\nc\cL{{\mathcal L}}
\nc\bfm{{\boldsymbol m}}\nc\bfM{{\boldsymbol M}}\nc\cM{{\mathcal M}}
\nc\bfn{{\boldsymbol n}}\nc\bfN{{\boldsymbol N}}\nc\cN{{\mathcal N}}
\nc\bfo{{\boldsymbol o}}\nc\bfO{{\boldsymbol O}}\nc\cO{{\mathcal O}}
\nc\bfp{{\boldsymbol p}}\nc\bfP{{\boldsymbol P}}\nc\cP{{\mathcal P}}
\nc\bfq{{\boldsymbol q}}\nc\bfQ{{\boldsymbol Q}}\nc\cQ{{\mathcal Q}}
\nc\bfr{{\boldsymbol r}}\nc\bfR{{\boldsymbol R}}\nc\cR{{\mathcal R}}
\nc\bfs{{\boldsymbol s}}\nc\bfS{{\boldsymbol S}}\nc\cS{{\mathcal S}}
\nc\bft{{\boldsymbol t}}\nc\bfT{{\boldsymbol T}}\nc\cT{{\mathcal T}}
\nc\bfu{{\boldsymbol u}}\nc\bfU{{\boldsymbol U}}\nc\cU{{\mathcal U}}
\nc\bfv{{\boldsymbol v}}\nc\bfV{{\boldsymbol V}}\nc\cV{{\mathcal V}}
\nc\bfw{{\boldsymbol w}}\nc\bfW{{\boldsymbol W}}\nc\cW{{\mathcal W}}
\nc\bfx{{\boldsymbol x}}\nc\bfX{{\boldsymbol X}}\nc\cX{{\mathcal X}}
\nc\bfy{{\boldsymbol y}}\nc\bfY{{\boldsymbol Y}}\nc\cY{{\mathcal Y}}
\nc\bfz{{\boldsymbol z}}\nc\bfZ{{\boldsymbol Z}}\nc\cZ{{\mathcal Z}}

\nc{\remove}[1]{}

\def\h_q{\qopname\relax{no}{h_q}}

\newcommand\reals{{\mathbb R}}

\newcommand\pr[1]{{\mathbb P}\big\{#1\big\}}
\newcommand\ep[1]{{\mathbb E}\big[#1\big]}

\newcommand{\algrule}[1][.2pt]{\par\vskip.5\baselineskip\hrule height #1\par\vskip.5\baselineskip}
\newcounter{ALC@tempcntr}
\DeclareMathAlphabet{\mathpzc}{OT1}{pzc}{m}{it}

\title{Associative Memory using Dictionary Learning and \\ Expander Decoding}

\author[$\dagger$]{Arya Mazumdar}
\author[$\ddagger$]{Ankit Singh Rawat\thanks{This work was done when the author was with the Computer Science Department, Carnegie Mellon University, PA, USA.}} %
\affil[$\dagger$]{College of Information \& Computer Science, University of Massachusetts Amherst, MA, USA}
\affil[$\ddagger$]{Research Laboratory of Electronics, Massachusetts Institute of Technology, MA, USA\quad \quad \quad {E-mail:~{arya@cs.umass.edu},~{asrawat@mit.edu}} }

\begin{document}

\maketitle
\begin{abstract}
An associative memory is a framework of content-addressable memory that  stores a collection of message vectors (or a {\em dataset}) over a neural network while enabling a neurally feasible mechanism to recover any message  in the dataset from its noisy version. Designing an associative memory requires addressing two main tasks: 1) {\em learning phase}: given a dataset, learn a concise representation of the dataset in the form of a graphical model  (or a neural network),  2) {\em recall phase}: given a noisy version of a message vector from the dataset, output the correct message vector via a neurally feasible algorithm over the network learnt during the learning phase. This paper studies the problem of designing a class of neural associative memories which learns a  network representation  for a large dataset that ensures correction against a large number of adversarial errors during the recall phase. Specifically, the associative memories designed in this paper can store dataset containing $\exp(n)$ $n$-length message vectors over a network with $O(n)$ nodes and can tolerate $\Omega(\frac{n}{{\rm polylog} n})$ adversarial errors. This paper carries out this memory design by mapping the learning phase and recall phase to the tasks of dictionary learning with a square dictionary and iterative error correction in an expander code, respectively.
\end{abstract}

\section{Introduction}
\label{sec:intro}

Associative memories aim to address a problem that naturally arises in many information processing systems: given a dataset $\cM$ which consists of $n$-length vectors, design a mechanism to concisely store this dataset so that any future query corresponding to a noisy version of one of the vectors in the dataset can be mapped to the correct vector. An associative memory based solution to this problem is broadly required to have two key components: 1) dataset must be stored in the form of a neural network (graph) and 2) the mechanism to map a noisy  query to the associated valid vector should be implementable in an iterative neurally feasible  manner over the  network (a
{\em neurally feasible} algorithm  employs only local computations at the nodes of the
corresponding  network based on the information obtained from their neighboring nodes). The tasks of learning the graph representation from the dataset and mapping erroneous vectors to the associated correct vectors are referred to as {\em learning phase} and {\em recall phase}, respectively. 

The overarching goal of designing an associative memory that can store a large dataset (ideally containing $\exp(n)$ message vectors using a neural network with $O(n)$ nodes) while ensuring robustness to a large number of errors (ideally $\Omega(n)$ errors) during the recall phase has led to multiple research efforts in the literature. The binary Hopfield networks, as studied in \cite{Hop82,MPRV87}, provide one of the earliest designs for the associative memories. Given a dataset containing binary vectors from $\{\pm 1\}^n$, Hopfield networks learn this dataset in the form of an $n$-node weighted graph by employing Hebbian learning~\cite{Heb05}, i.e., the weighted adjacency matrix of the graph is defined by  summing the outer products of all message vectors in the dataset. However, in their most general form, these networks suffer from small capacity. In \cite{MPRV87}, McEliece et al. show that these networks can only store $O\big(\frac{n}{\log n})$ message vectors when these messages correspond to arbitrary $n$-length binary vectors and the recall phase is required to tolerate linear $\Omega(n)$ random errors. This has motivated the researchers to look at various generalizations of Hopfield networks (see,  \cite{gross,JLZ,mceliece,MGZ,tanaka} and references therein). However, these solutions again fail to simultaneously achieve both large capacity and error tolerance. 

One remedy to small capacity is to design associative memories with structural assumptions on the dataset. This approach has been considered in \cite{gripon2011sparse,HilTran15,karbasi,kumar11,MR15,SalKarb12}. In particular in \cite{gripon2011sparse}, Gripon et al. store a dataset comprising $O(n^2)$ sparse vectors in the form of cliques in a neural network. In \cite{HilTran15}, Hillar and Tran design a Hopfield network with $n$ nodes that can store $\sim 2^{\sqrt{2n}}/n^{1/4}$ message vectors
and is robust against $n/2$ {\em random} errors.
 In \cite{karbasi,kumar11,SalKarb12,MR15}, the message vectors that need to be stored are assumed to constitute a subspace. In \cite{karbasi,kumar11,SalKarb12}, the task is to learn a bipartite factor graph of the linear constraints satisfied by the dataset subspace. The error correction during recall phase is then performed by running a belief propagation algorithm~\cite{MCT08} over the bipartite graph. In \cite{karbasi}, Karbasi et al. work with a model where the message vectors in the dataset have overlapping sets of coordinates so that shortened vectors obtained by restricting the original message vectors to each of these overlapping sets belong to a subspace. Under this model, they design associative memories that can store exponential number (in $n$) of message vectors while correcting linear number (in $n$) of random errors during the recall phase. 

The results in \cite{karbasi} hinge on the fact that the learning phase of their memory design recovers a bipartite graph which has certain desirable structural properties that are required for belief propagation type decoders to converge. However, no guarantee of recovering such a bipartite graph during the learning phase is provided in \cite{karbasi} even when we assume the subspace associated with the dataset has one such graphical representation to begin with. Recognizing the requirement of learning correct bipartite graph during the learning phase, Mazumdar and Rawat explore a sparse recovery based approach to design associative memories with the subspace dataset model in \cite{MR15}. This approach assumes that the dataset belongs to a subspace whose orthogonal subspace has null space property, a sufficient condition for sparse signal recovery. This allows one to learn any basis for the orthogonal subspace during the learning phase and then recast the recall phase as a sparse recovery problem~\cite{CandesTao2006,donoho2006compressed}. The approach in \cite{MR15} also allows for the strong error model containing {\em adversarial} errors. Specifically, \cite{MR15}  considers two candidate signal models which contain $n$-length message vectors and utilize $O(n)$ sized neural networks to store the signals. The two models have the datasets of sizes $\exp(n^{3/4})$ and $\exp(r)$  with $1 \leq r \leq n$, respectively. Furthermore, the designed associative memories based on these two signal models respectively allow for recovery from $\Omega(n^{1/4})$ and $\Omega\left(\frac{n-r}{\log^6 n}\right)$ adversarial errors in a neurally feasible manner.

In this paper, we also follow the subspace model as in \cite{karbasi,kumar11,SalKarb12,MR15}. We assume the dataset to form a subspace which is defined by {\em sparse linear constraints}. 
The model of sparse linear constraints are quite natural and less restrictive than the previous models of works such as \cite{MR15}.
Note that this signal model is similar to the model explored in Karbasi et al. \cite{karbasi}. However, our approach and contributions differ from  \cite{karbasi}, as we ensure that the learning phase {\em provably} generates the correct bipartite graph which can guarantee the error correction from a large number of errors using an iterative algorithm during the recall phase. We also note that similar to \cite{MR15} we work with the stronger error model involving adversarial errors, but our scheme is superior to that of \cite{MR15}  in terms of storage capacity (see, Theorems~\ref{thm:design_eff},~\ref{thm:bin_eff}) and number of correctable adversarial errors (improvement by poly-log factors, see, Theorems~\ref{thm:design_eff}, \ref{thm:design_noneff},~\ref{thm:bin_eff}). We  want to point out that the main technical challenge in associative memory is not to individually design the learning or recall phases, but to interface them in a way that is consistent with the operations of both phases, and to give an end-to-end performance guarantee.  

{Here, we note that the problem of designing an associative memory is closely related to the well studied nearest neighbor search (NNS) problem and its relaxation  approximate nearest neighbor search (A-NNS) problem~\cite{IM98,AI08,Sam05,WSSJ14}. The solutions to the A-NNS problem  enable one to store a dataset in such a manner that noisy versions of the vectors in a dataset (with bounded noise) can be mapped to the correct vectors. 
Additionally, the A-NNS solutions do not put assumptions on the dataset. 
However, this comes at the cost of removing the requirement of having a fast iterative or neurally feasible recall phase. Furthermore, the A-NNS solutions, especially based on locally sensitive hashing~\cite{IM98,HPM12} have large space complexity, i.e., polynomial in size of dataset.  We note that the A-NNS solutions are very much  aligned to the vector (image) retrieval task~\cite{JDS11,YGJJ15,FGJ16} which need not have a neurally feasible retrieval algorithm.}

The rest of the paper is organized as follows. In Sec.~\ref{sec:prelim}, we define the dataset model considered in this paper and present the main results of this paper along with key techniques and ideas involved in establishing those results. 
Sec. \ref{sec:proof} is dedicated to the proof of the main theorem.
In Sec.~\ref{sec:learning}, we describe the learning phase of the associative memory design results along with the relevant technical details. In Sec.~\ref{sec:recall}, we present an iterative error correction algorithm which is employed during the recall phase of the designed associative memory. This analysis of the algorithm relies on the expansion properties of the bipartite graph which defines the dataset and is learnt during the learning phase. We conclude the paper with some comments on performance   in Sec.~\ref{sec:simulations}.

\section{Main results and techniques}
\label{sec:prelim}

\subsection{Model for datasets}
\label{sec:sys}
 We focus on the associative memories based on the operations on $\reals$, the set of real numbers. 
 In our first model, we consider the message patterns to be vectors over $\reals$.
 In the second model we comment on neural associative memories storing binary message patterns that are obtained by our approach. 
 
 \subsubsection{Dataset over real numbers: the sparse-sub-Gaussian model}\label{sec:real}
 We assume the message set to form a linear subspace defined by sparse linear constraints over $\reals$.  Let $\cM \subseteq \reals^n$ denote the set of message vectors (signals) that need to be stored on the associative memory. Let $B$ be an $m \times n$ matrix comprising the linear constraints that define the message set $\cM$. In particular, we have 
\begin{align}\label{eq:cond1}
B\bfx  = 0~~~~\forall~\bfx = (x_1, x_2,\ldots, x_n) \in \cM.
\end{align}
In order to fully specify the message set $\cM$, we still need to provide a stochastic model for the matrix $B$. Towards this, we consider a random ensemble of sparse matrices. For each $j \in [n]:= \{1, 2,\ldots, n\}$, we  consider the following experiment. We pick $d$ element uniformly at random with replacement from  the set $[m] $. Let $\cN_j$ denote the set comprising these randomly picked elements. For $1 \leq i \leq m,~1\leq j \leq n$, we define
\begin{align}
\xi_{i, j} = \begin{cases}
1& \mbox{if}~i \in \cN_j \subset [m] \\
0& \mbox{otherwise}.
\end{cases}
\end{align}
Let $\big\{R_{i, j}\big\}_{1 \leq i \leq m,~1\leq j \leq n}$ be a collection of independent and identically distributed (i.i.d.) sub-Gaussian random variables. Given the random variables, $\big\{\xi_{i,j}, R_{i, j}\big\}_{1 \leq i \leq m,~1\leq j \leq n}$,  we assume that the $(i, j)$-th entry of the matrix $B$ is defined as 
\begin{align}
\label{eq:Bdef}
B_{i, j} = \xi_{i, j}R_{i, j} \in \reals~~~~\text{for}~1 \leq i \leq m,~1\leq j \leq n.
\end{align}
Through out this paper, we refer to this model for the dataset to be stored on a neural associative memory as {\em sparse-sub-Gaussian model}. {We work with various values of $d$ which we specify while stating different parameters that we obtain for the designed associative memories in Sec.~\ref{sec:main}.

This model is a quite natural random model of bipartite graphs that allow for multi-edges. Indeed, consider a  bipartite graph with disjoint sets of vertices $[n]$ (variable nodes) and $[m]$ (check nodes). There are   $d$ edges out of each variable node, being incident on uniformly and independently chosen vertices from the check nodes.

\begin{remark}
The requirement on $R_{i, j}$ is quite generic as it allows for many distributions. For example, we can assume that $R_{i, j}$ belongs to a finite set of integers $\{-L, -L + 1,\ldots,-1, 1, \ldots, L - 1, L\}$. Similarly, in another setup, $R_{i, j}$ can be assumed to be a Gaussian random variable.
\end{remark}

\subsubsection{Binary dataset}\label{sec:bin}
Our model of binary dataset is same as above except for the fact that 1) $\cM \subseteq \{+1,-1\}^n$, and 2) $R_{i,j}$ is uniform over $\{+1, -1\}$ in \eqref{eq:Bdef}. The condition of \eqref{eq:cond1} must be satisfied for any $\bfx \in \cM$.

\subsection{Our main results}
\label{sec:main}
We establish that, for a dataset $\cM$ corresponding to the null-space defined by the matrix $B$, the said matrix $B$ can be exactly recovered  from the dataset in polynomial time. Recall that there can be many sets of basis-vectors for the null-space of $\cM$. Still, we claim that it is possible to accurately recover the matrix $B$ that has been generated by the sparse-sub-Gaussian model described above.

It is essential for us that we recover the matrix $B$ exactly. Being generated by the random model defined above, $B$
exhibits certain graph expansion property that is necessary for our recall phase to be successful.
This matrix $B$  enables the error correction during the recall phase with the help of a simple iterative (neurally feasible algorithm). We summarize the parameters achieved by such memory as follows.

\begin{theorem}
\label{thm:design_eff} 
Suppose that $c, c', c''>0$ are three constants. Let $n$ be a large enough integer and $m = c\frac{n}{\log n}$. Assume that $B$ is an $m \times n$ matrix generated from the sparse-sub-Gaussian model described in Sec.~\ref{sec:real} with   $c'\le d \le c''\log n$,  and $\Mc = \{\bfx\in \reals^n:~B\bfx = 0\}$. Then, with high probability (w.h.p.) $\cM$ is an $n-m = n(1 - c/\log n)$ dimensional subspace that can be stored in a neural network (learned in poly-time in the learning phase) while allowing for correct recovery from $\Omega(\frac{n}{d^2\log^2 n})$ adversarial errors during the recall phase with a neurally feasible algorithm. 
\end{theorem}
The proof of this theorem has been provided in Sec. \ref{sec:proof}.
This result is obtained by utilizing a novel connection between recovering the matrix $B$ defining the underlying dataset $\Mc$ and the dictionary learning problem with a square dictionary as studied in \cite{SWW12,A16,BN16}. Given access to the dataset $\Mc$, we can easily find a basis for the null-space of $\cM$ containing $m = n- \dim(\Mc)$ $n$-length vectors. Let $A$ denote the $m \times n$ matrix which has the $m$ vectors in this basis as its rows. 
Note that the row vectors of $B$ also span the subspace orthogonal to the dataset $\Mc$. Moreover, w.h.p., $B$ is a full rank matrix. This implies that the following relationship holds w.h.p.,
\begin{align}
A = DB,
\end{align}
where $D$ is an invertible $m \times m$ matrix. Note that recovering the matrix $B$ from $A$ is now equivalent to dictionary learning problem~\cite{OlsField} where $n$ columns of $A$ and $B$ corresponds to $n$ observations and the associated coefficients, respectively. Furthermore the matrix $D$ corresponds to a square dictionary~\cite{SWW12}. 

As for the recall phase, we rely on the observations (as shown in Sec.~\ref{sec:recall}) that w.h.p. the bipartite graph associated with the sparse random matrix $B$ is an expander graph. Assume that we are given a noisy version $\bfy$ of a valid message vector $\bfx \in \Mc$ such that we have
\begin{align}
\bfy  = \bfx + \bfe
\end{align}
where $\bfe$ denotes the error vector. Recovering $\bfx$ from the observation $\bfy$ can be cast as a sparse recovery problem of recovering $\bfe$ from 
\[
\bfz = B\bfy = B(\bfx + \bfe) = B\bfe.
\]
If the bipartite graphs associated with $B$ is an expander graph (which holds w.h.p.), we can solve this sparse recovery problem by an efficient and iterative algorithm~\cite{JXHC09} which is motivated by the decoding algorithm of expander codes~\cite{SS96} in coding theory literature.

Due to the sample complexity requirements for efficient square-dictionary learning algorithms~\cite{SWW12,A16,BN16}, the above model allows us to store datasets that satisfy at most $O\big(\frac{n}{\log n}\big)$ linear constraints. However if we allow for a learning-phase that takes quasi-polynomial time, then it is possible to store restricted datasets that satisfy $m= \Theta(n)$ sparse-linear constraints. We summarize the result below.

\begin{theorem}
\label{thm:design_noneff}
Let $n$ be a large enough integer and $m = cn$ for a some constant $c < 1/200$. For a large enough constant $C > 0$, let $B$ be an $m \times n$ matrix generated from the sparse-sub-Gaussian model described in Sec.~\ref{sec:real} with $d = C\log n$ and  $\Mc = \{\bfx\in \reals^n:~B\bfx = 0\}$. 
Then w.h.p., $\cM$ is an $n-m = n(1 - c)$ dimensional subspace that can be stored in a
neural network (learned in quasi-polynomial-time in the learning phase) while allowing for error correction from $\Omega(\frac{n}{\log^2 n})$ adversarial errors during the recall phase with a neurally feasible algorithm.
\end{theorem}
While in terms of storage capacity this theorem is inferior to that of Theorem \ref{thm:design_eff},  it may represent some datasets better, and has better error correction capability. While the recall phase of this algorithm works same as above, for the learning phase we can no longer rely on the dictionary-learning algorithms. Instead we do an exhaustive search over all possible sparse vectors to find out a sparse basis for the null-space of $\cM$ which end up taking a quasi-polynomial time, if we choose parameters suitable for the recall phase. We here crucially use the fact
that for $m = cn$ and $d = C\log n$ such a sparse basis is unique, which can be obtained from the results of \cite{SWW12}. The proof of the recall phase for this theorem remains same as that of Theorem \ref{thm:design_eff}.

Finally, while both Theorems \ref{thm:design_eff} and \ref{thm:design_noneff} have their counterparts when storing binary vectors,
we present only one result for brevity. A sketch of the proof of the following theorem has been given in Sec. \ref{sec:binary}.
\begin{theorem}[Binary dataset]
\label{thm:bin_eff}
Suppose that $c, c', c''>0$ are three constants. Let $n$ be a large enough integer such that $m = c\frac{n}{\log n}$. Assume that $B$ is an $m \times n$ matrix generated from the binary dataset model described in Sec.~\ref{sec:bin} with $c' \leq  d  \leq c''\log n$ and  $\Mc = \{\bfx\in \{\pm1\}^n:~B\bfx = 0\}$. Then w.h.p., $|\cM|= \exp(n-\alpha n \log(d \log n)/\log n)$ for a constant $\alpha$ and $\cM$ can be stored in a
neural network (learned in polynomial-time in the learning phase) while allowing for error correction from $O(\frac{n}{d^2\log^2 n})$ adversarial errors during the recall phase with a neurally feasible algorithm.
\end{theorem}

\section{Proof of Theorem \ref{thm:design_eff}}\label{sec:proof}
\subsection{Learning phase of associative memory design}
\label{sec:learning}

As discussed in the previous section, under the dataset model considered in this paper, the learning phase of the associative memory design can be mapped to the problem of dictionary learning with a square dictionary. The very same dictionary learning problem with slightly different random model for the coefficient vector has been studied in~\cite{SWW12,A16,BN16}. In Appendix~\ref{sec:dictionary_background}, we briefly describe this line of work along with the results that are used in this paper. We then utilize the dictionary learning algorithm used in \cite{A16} to exactly learn the matrix $B$ which define our dataset and comment on the modifications required in the analysis of Adamczak~\cite{A16} to obtain guarantees on the performance of this algorithm. 

\subsubsection{Exact recovery of the matrix $B$}
\label{sec:learningB}
Our learning phase constitutes learning the matrix $B$ exactly from the dataset $\cM$.
Utilizing the dictionary learning algorithm from \cite{A16}, we design the learning phase for an associative memory storing the message set described in Sec.~\ref{sec:sys}. The learning phase consists of the following two steps.
\begin{enumerate}
\item Given the message vectors from the dataset $\cM$, first construct a basis for the subspace orthogonal to the dataset subspace $\cM = \{\bfx~:~B\bfx = 0\} \subset \RR^n$ with $\dim(\Mc) = n - m$.
\item Let $A \in \reals^{m \times n}$ denote the basis obtained in the previous step. Since w.h.p. $B$ is a full-rank matrix, we have 
\[
A = DB,
\]
where $D \in \reals^{m \times m}$ is a non-singular matrix. Now employ the modified ERSpUD dictionary learning algorithm~\cite{A16} with the matrix $A$ as its input. Note that the algorithm outputs candidates for the matrices $D$ and $B$. The method of this square-dictionary learning
and the algorithm are summarized in Appendix~\ref{sec:dictionary_background}.
\end{enumerate}

Next, we show that the proposed learning phase w.h.p. exactly recovers the matrix $B$. Note that the sparse-sub-Gaussian model used to generate $B$ (cf.~Sec.~\ref{sec:sys}) slightly differs from the Bernoulli-sub-Gaussian model studied in \cite{SWW12,A16} (cf.~Appendix~\ref{sec:dictionary_background}). In particular, for every $j \in [n]$, the distribution of the random variables $\{\xi_{i,j}~:~i \in [m]\}$ and $\{\eta_{i, j}~:~i \in [m]\}$ is different\footnote{We focus on the sparse-sub-Gaussian model as opposed to the Bernoulli-sub-Gaussian model as the bipartite graph associated with the matrix $B$ generated by the sparse-sub-Gaussian model is a good expander w.h.p. We utilize this fact while designing the recall phase for the proposed associative memory in Sec.~\ref{sec:recall}.}. However, this difference is not very crucial for the success of the learning algorithm as we still have independence among the random variables  $\xi_{i,j}$s which are indexed by different values of $j \in [n]$.  We formalize the exact recovery guarantees for the matrix $B$ in the following result.

\begin{theorem}
\label{thm:learning}
Let $B \in \reals^{m \times n}$ be a matrix generated by the sparse-sub-Gaussian model (cf.~Sec.~\ref{sec:sys}) and $\Mc$ be the associated dataset, i.e., $\Mc = \{\bfx~:~B\bfx = 0 \}$. Then there exists a constant $c  > 0$ such that whenever we have $n \geq c m\log m$ the two step learning phase of the associative memory as described above exactly recovers the linear constraints in $B$ with probability at least $1 - 1/n$.
\end{theorem}

We refer the reader to Appendix~\ref{appen:thm_learning} for the proof of Theorem \ref{thm:learning}.

\subsection{Recall phase of associative memory design}
\label{sec:recall}

In this section we present an iterative algorithm which recovers the correct message vector among the dataset $\Mc$ from its noisy version. The noisy observation is assumed to be corrupted at  adversarially chosen coordinates. The correctness of the iterative algorithm relies on the observation that the bipartite graph associated with the matrix $B$ which defines our dataset $\Mc$ is a good expander graph. We first formalize this expansion property in the following result. We then present the iterative algorithm and show that it can provably tolerate $\Omega\big(\frac{n}{{\rm polylog} n}\big)$ adversarial errors.

\subsubsection{Expansion property of the bipartite graph defined by $B$}
\label{sec:expanderB}

Let $\cG_{B} = (\Lc = [n], \Rc = [m], \cE_{B})$ be a bipartite graph where $\Lc$ and $\Rc$ denote the index sets of left and right vertices, respectively. The matrix $B$ which defines our dataset $\Mc$ gives the $m \times n$ adjacency matrix of the graph $\Gc$, i.e., for $\ell \in \Lc$ and $r \in \Rc$, we have an edge $(\ell, r) \in \Ec_{B}$ iff $B_{r, \ell} \neq 0$. More specifically, the weight of the edge $(\ell, r) \in \Ec_{B}$ is $w_{\ell, r} = B_{r, \ell}$. It follows from the sparse-sub-Gaussian model (cf.~Sec.~\ref{sec:sys}) which generates the random matrix $B$ that every vertex in $\cL$ has degree $d$ and each of the $d$ neighbors for a  vertex in $\Lc$ are chosen uniformly at random from the set of right vertices $\Rc$ with replacement. The following result states that expansion properties that hold for such a graph with high probability.

\begin{proposition}
\label{prop:expander}
Assume that $\epsilon > 0$ and $d = O(\frac{n}{m\log n})$. Let $\cG = (\Lc, \Rc, \cE)$ be a random $d$-left regular graph where each of the $d$ neighbors for a left vertex are chosen uniformly at random from the set of right vertices with replacement. Then, for a large enough $n$, w.h.p., $\cG$ is an $\Big(\frac{m^2}{d^2 n}, (1 - \epsilon)d\Big)$-expander graph, where a bipartite graph is  $(t,l)$-expander, if for every $\cS \subseteq \Lc$ such that $|\cS| \leq t$, we have 
$|\cN(\cS)| \geq l|\cS|.$ 
Here, $\cN(\cS) \subseteq \Rc$ denotes the  vertices in $\cR$ that are neighbors of  vertices in $\cS$.
\end{proposition}
\begin{proof}
Let's consider a set $\cS \subseteq \Lc$ such that $|\cS| = s \leq \frac{m^2}{d^2 n}$. Let $\cT \subseteq \Rc$ be a set of right vertices such that $|\cT| < (1 - \epsilon)ds$. The probability that $\cN(\cS) \subseteq \cT$ is upper bounded by 
$\left(\frac{(1 - \epsilon)ds}{m}\right)^{ds}.$
Now, taking the union bound over all the sets $\cS \subseteq \Lc$ such that $|\cS| = s$ and the sets $\cT \subseteq \Rc$ such that $|\cT| < (1 - \epsilon)ds$, the probability $P_s$ that the the graph $\cG$ has a non-expanding set of size $s$,  is upper bounded as follows. 
\begin{align}
\label{eq:expan1}
P_s &\leq {n \choose s}{m \choose (1 - \epsilon)ds}\left({(1 - \epsilon)ds}/{m}\right)^{ds} & \nonumber \\
& \leq e^{s + (1 - \epsilon)ds}\left({n}/{s}\right)^{s}\left({(1 - \epsilon)ds}/{m}\right)^{\epsilon ds}.
\end{align}
We can rewrite \eqref{eq:expan1} as, 
\begin{align}
P_s &\leq  
e^{s + (1 - \epsilon)ds}\left({d n}/{m}\right)^s\left({ds }/{ m}\right)^{\epsilon ds - s}.
\end{align}
Now, using our assumption that $s \leq \frac{m^2}{d^2 n}$, we obtain that 
\begin{align}
\label{eq:expan2}
P_s &\leq 
e^{s + (1 - \epsilon)ds}\left({m}/{d n}\right)^{\epsilon d s - 2s}.
\end{align}
Using union bound, 
we have that $\cG$ is not an $\Big(\frac{m^2}{d^2 n}, (1 - \epsilon)d\Big)$-expander with probability at most
\begin{align}
\label{eq:final_expan}
\sum_{s = 1}^{\frac{m^2}{d^2 n}} P_s \leq  \frac{m^2}{d^2 n} e^{s + (1 - \epsilon)ds}\left(\frac{m}{d n}\right)^{\epsilon d s - 2s}.
\end{align}
Now, for large enough $n$, the R.H.S. of \eqref{eq:final_expan} vanishes as we have $\frac{m}{dn} = O(\frac{1}{\log n})$
\end{proof}

\subsubsection{Iterative decoding algorithm}
\label{sec:iterative}

Remember that during the recall phase we are given an $n$-length observation vector $\bfy$ which is noisy version of one of the message vectors from the dataset $\Mc$, i.e.,
\begin{align}
\bfy  = \bfx + \bfe,~\text{for some}~\bfx \in \Mc.
\end{align}
Assuming that we have exactly learnt the $m \times n$ matrix $B$ during the learning phase of the associative memory (as described in Sec.~\ref{sec:learning}), we obtain an $m$-length vector as follows. 
\begin{align}
\bfz = B\bfy  = B(\bfx + \bfe) = B\bfe,
\end{align}
where the last equality follows as we have $\bfx \in \Mc = \{\bfx \in \reals^n~:~B\bfx = 0\}$. Note that we have reduced the problem of recovery of the correct message vector $\bfx$ from $\bfy$ to the task of recovering $\bfe$ from $\bfz$. Assuming that the error vector $\bfe$ satisfies certain sparsity constraint, the latter problem is exactly the problem of recovering the sparse vector $\bfe$ from its linear measurements via the measurement matrix $B$. As shown in Proposition~\ref{prop:expander}, w.h.p., the matrix $B$ corresponds to the adjacency matrix of an expander graph. In \cite{JXHC09}, Jafarpour et al. have adapted the iterative error correction algorithm for expander codes from \cite{SS96} to the problem of sparse recovery problem when the measurement matrix corresponds the adjacency matrix of a good expander graph. Here we propose to employ this iterative algorithm to recover $\ev$ from $\bfz$. The algorithm requires calculation of {\em gap} for each of the linear constraints defined by the matrix $B$ (or rows of the matrix $B$) which we formally define below. 

\begin{definition}
Let $\bfe$ be an error vector and $\bfz = B\bfe$. Given an estimate $\widehat{\bfe}$ for $\bfe$, for each linear constraint indexed by $i \in [m]$, we define a gap $g_i$ as follows. 
\begin{align}
\label{eq:gap_val}
g_i = z_i - \sum_{j = 1}^nB_{i,j}\widehat{e}_j.
\end{align}
\end{definition}

We describe the algorithm in Fig.~\ref{fig:iterative} and present the theoretical guarantees for the performance of the algorithm from \cite{JXHC09} as follows.

\begin{figure}[t!]
\algrule[1pt]
\textbf{Expander decoding algorithm}
\algrule[1pt]
\begin{algorithmic}[1]
\REQUIRE The vector $\bfz = B\bfe$ and the matrix $B$.
\STATE Define $\Nc_j := \{i \in [m]~:~B_{i,j} \neq 0 \}$~$\forall~j \in [n]$.
\STATE Initialize $\widehat{\bfe} = 0$.
\IF{$\bfz = B\widehat{\bfe}$}
\STATE End the decoding and output $\widehat{\bfe}$.
\ELSE 
\STATE Find an index $j \in [n]$ such that the multiset $\{\frac{g_{i}}{B_{i,j}}\}_{i \in \Nc_j}$ has at least $(1 - 2\epsilon)d$ identical elements, say $\delta$. Here, $g_i$ is the gap (cf.~\eqref{eq:gap_val}) of the constraint defined by the $i$th row of $B$.
\STATE Set $\widehat{e}_j \leftarrow \widehat{e}_j + \delta$ and go to 2.
\ENDIF
\end{algorithmic}
\algrule[1pt]
\caption{Recovery algorithm for sparse vector from expander graphs based measurement matrix~\cite{JXHC09}.}
\label{fig:iterative}
\end{figure}

\begin{proposition}[\cite{JXHC09}]
\label{prop:iterative}
Let $B$ be an $m \times n$ matrix which is the adjacency matrix for a $(2k, (1 - \epsilon)d)$ expander bipartite graph with $\epsilon \leq \frac{1}{4}$.
Then, given the measurement vector $\bfz = B\bfe$ for any $k$-sparse vector $\bfe$, the expander decoding algorithm (cf.~Fig.~\ref{fig:iterative}) successfully recovers $\bfe$ in at most $2k$ iterations.
\end{proposition}

We now employ Proposition~\ref{prop:iterative} to characterize the error correction performance of the designed associative memories during the recall phase. 

\begin{theorem}
\label{thm:recall}
Let $B$ be the $m \times n$ matrix generated by the sparse-sub-Gaussian model described in Sec.~\ref{sec:sys} and $\Mc$ denote the dataset associated with the matrix $B$. Then, with probability at least $1 -o(1)$, the recall phase based on the iterative decoding algorithm described in Fig.~\ref{fig:iterative} can correct at least $\frac{m^2}{2d^2 n}$ adversarial errors.
\end{theorem}

\begin{proof}
It follows from Proposition~\ref{prop:expander} that with probability at least $1 - o(1)$, the matrix $B$ corresponds to the adjacency matrix of an $\Big(\frac{m^2}{d^2 n}, (1 - \epsilon)d\Big)$-expander graph. Combining the expansion parameters for this expander graph with the result in Proposition~\ref{prop:iterative}, we obtain that the iterative decoding algorithm (cf.~Fig.~\ref{fig:iterative}) can recover the error vector $\bfe$ from $\bfz = B\bfe$ as long as $\bfe$ has at most $\frac{m^2}{2 d^2 n}$ non-zero coordinates. Given $\bfy$ and $\bfe$, it is straightforward to obtain the correct message vector as $\bfx = \bfy - \bfe$. This completes the proof.
\end{proof}

\section{Proof sketch of Theorem \ref{thm:bin_eff}: Associative memory storing binary vectors}
\label{sec:binary}
Since the graph defined by $B$ is still an expander (with edge weights $\{+1, -1\}$), for the recall phase we rely
on the same expander decoding algorithm. We just want to guarantee that  $|\cM| = |\{\bfx \in \{\pm 1\}^n: B\bfx =0 \}|$
is of size about $\exp(n-\alpha n \log(d \log n)/\log n)$ w.h.p. The algorithm to learn $B$ is same as that of Theorem \ref{thm:design_eff}.

Instead of the random model that we have considered in Sec.~\ref{sec:bin}, consider a random matrix $B \in \{+1, 0, -1\}^{m \times n}$ whose
each row has independently and uniformly chosen $d'$ nonzero ($\{+1, -1\}$) values. This model allows us to come up with a straight-forward analysis of number of binary vectors in the null-space, while the original model gives the same estimate but with significantly lengthier analysis, that we omit for the interest of space. Note that $d' \sim d\frac{n}{m}$ w.h.p. Now for a randomly and
uniformly chosen $\pm1$ vector $\bfy$ of length $n$, and for some constant $c'>0,$
\[
\pr{B\bfy =0} = \left({\binom{d'}{\frac{d'}{2}}}/{2^{d'}}\right)^m \ge \Big(\frac{1}{c'd'}\Big)^{m/2}.
\]
This means  $\ep{|\cM|} \ge 2^{n}\cdot \Big({1}/{(c'd')}\Big)^{m/2} = 2^{n - \frac{m}{2}\log (c'd')}$. Substituting, $m  = c\frac{n}{\log n}$, we get the promised size of $\cM$.

\section{Simulation results}
\label{sec:simulations}

Though our main contribution is theoretical,
in this section we evaluate the proposed associative memory on synthetic dataset to verify if our methods works. Only a representative figure is presented here (Fig.~\ref{fig:error_recall}). We consider three sets of system parameters $(m, n, d)$ for the dataset to be stored. For each set of parameters, we first generate an $m \times n$ random matrix $B$ according to the sparse-sub-Gaussian model (cf. Sec.~\ref{sec:sys}). Each non-zero entry of the matrix $B$ is drawn uniformly at random from the set $\{\pm 1, \pm 2, \pm 3\}$. We then generate multiple message vectors which belong to the subspace  orthogonal to all the rows of the matrix $B$ and provide the learning phase with these vectors. Given these vectors we employ the dictionary learning based approach described in Sec.~\ref{sec:learningB} to obtain an estimate $\widehat{B}$ for the matrix $B$. As guaranteed by Theorem~\ref{thm:learning}, in our simulations, $\widehat{B}$ contains all the rows of the original matrix $B$ (however, in a different order). For all three sets of parameters under consideration, we then utilize the estimate $\widehat{B}$ to evaluate the performance of the expander decoding based recall phase (cf. Sec.~\ref{sec:recall}). For a fixed number $E$ of errors, we generate $100$ error vectors $\ev \in \RR^n$ with the number of non-zero entries in each error vector equal to $E$. The non-zero entries in these vectors are uniformly generated from the set $\{\pm 1,\ldots, \pm 4\}$. The positions of the non-zeros entries in each of these vectors are chosen according to a uniform random permutation on the set $[n]$. 

The performance of the recall algorithm in our simulations is illustrated in Fig.~\ref{fig:error_recall} where we plot the fraction of incorrectly recovered error vectors as we increase the number of errors. As expected from Theorem~\ref{thm:recall}, increasing $d$ while keeping $m$ and $n$ fixed degrades the performance of the recall phase. On the other hand, increasing $m$ while keeping $d$ and the ratio $\frac{m}{n}$ fixed improves the performance of the recall phase. \\

\begin{figure}[t!]
\centering
\includegraphics[width=6cm]{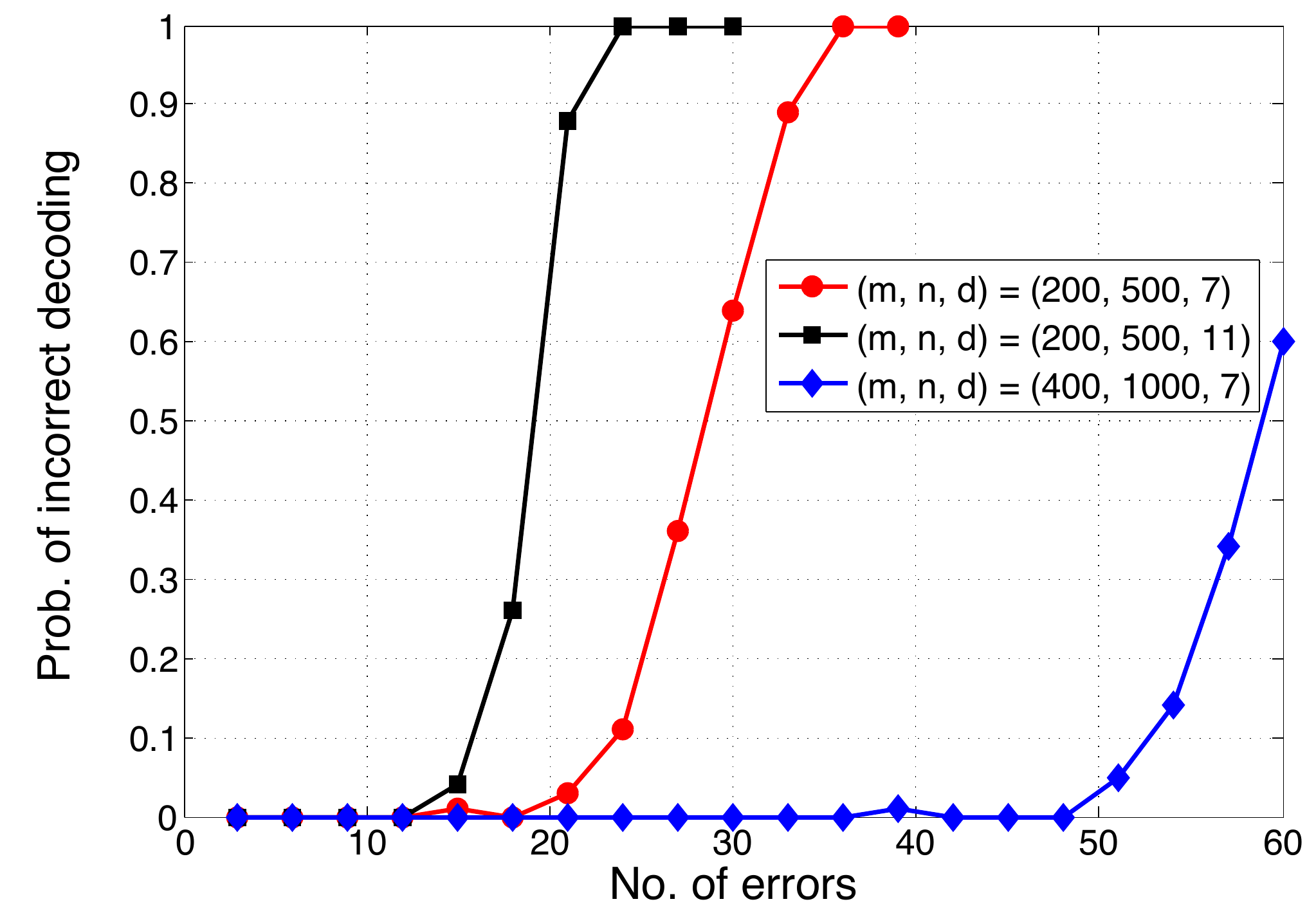}
\caption{\small Performance of recall phase for different sets of system parameters.}\label{fig:error_recall}
\end{figure} 

\noindent {\bf Concluding remarks} While we use dictionary learning as a tool in the learning phase, the model of our datasets are subspace models.
A large number of datasets on the other hand are also modeled by the {\em sparse dictionary} model (or union of subspaces). It is of interest to design associative memories, where the datasets are modeled as such. One other possible direction of future research would be to consider
a subspace model with a mixture of sparse and dense constraints, which potentially will be inclusive of larger classes of real datasets. For such datasets, under suitable assumption on the generative model, one can potentially employ the techniques of recovering planted sparse vectors in a subspace spanned by dense random sub-Gaussian vectors~\cite{DH14,QSW15} and utilize the recovered sparse constraints to design an iterative recall phase similar to the one presented in this paper. {As in the case of \cite{karbasi}, the networks (graphs) appearing in our associative memory design  share some similarities with the neural networks used for classification tasks. It is an interesting problem to further explore such connections.}

\newpage 

\bibliography{associative2016}
\bibliographystyle{plain} 

\newpage

\appendix

\begin{center}
{\large Appendix}
\end{center}

\section{The modified ER-SpUD algorithm and proof of Theorem~\ref{thm:learning}}
\label{sec:dictionary_background}
In \cite{SWW12}, Spielman et al. consider the following problem of exact dictionary learning. Let $D \in \reals^{m \times m}$ be an invertible matrix also referred to as the dictionary. Given $n$ observations
\begin{align}
\bfu_j = D\bfv_j~\text{for}~j \in [n], 
\end{align} 
the task is to exactly learn the dictionary $D$ and the coefficient matrix
\[
V = [\bfv_1, \bfv_2,\ldots, \bfv_n] \in \reals^{m \times n}.
\]
Spielman et al. assume that the coefficient vectors of the observation are randomly generate so that the entries of the coefficient matrix $V$ are independent and identically distributed~\cite{SWW12}. In particular, let 
\[
V_{i, j} = \eta_{i, j}R_{i, j},
\]
where $\eta_{i, j} \in \{0, 1\}$ and $R_{i, j} \in \reals$ are independent random variables. In particular, for some constant $\alpha$, they assume that
\begin{align}
\pr{\eta_{i,j} = 1} = 1 - \pr{\eta_{i, j}= 0} = \theta \in  \Big[\frac{2}{m}, \frac{\alpha}{\sqrt{m}}\Big],
\end{align}
and $R_{i, j}$ is a zero mean sub-Gaussian random variable such that
\[
\ep{|R_{i,j}|} \geq \frac{1}{10}~~~\text{and}~~~\pr{|R_{i,j}| \geq t} \leq 2\exp(-t^2/2).
\]

This random generative model for the coefficients $V$ is referred to as {\em Bernoulli-sub-Gaussian model}.
Under the Bernoulli-sub-Gaussian model, Spielman et al. show that the dictionary learning problem is well defined. In particular, as long as $n \geq \Omega(m\log m)$, for an alternative representation of the observations 
\[
U = [\bfu_1, \bfu_2,\ldots, \bfu_n] = D'V',
\]
where $A' \in \reals^{m \times m}$ is an invertible matrix and $V' \in \reals^{m \times n}$ is coefficient matrix with the per-column sparsity bounded by that of the original coefficient matrix $V$, we have 
\[
D' = D\Pi\Lambda
\]
and
\[
V' = \Lambda^{-1}\Pi V.
\]
Here, $\Lambda \in \reals^{m \times m}$ and $\Pi \in \reals^{m \times m}$ denote a diagonal matrix and a permutation matrix, respectively. This implies that for $n \geq \Omega(m\log m)$, any other representation of the observations which is explained by a square dictionary and the sparsest coefficient vectors have its dictionary and coefficient matrix as some permutation and scaling of the columns and rows of the original dictionary $D$ and the coefficient matrix $V$, respectively. Furthermore, Spielman et al. also present an algorithm for the exact dictionary learning problem that recovers the dictionary $D$ (up to scaling and permutations of the columns of $D$) and the coefficient matrix $V$  (up to scaling and permutations of the rows of $V$) provided that $n = O(m^2\log^2m)$ samples. Recently, Adamczak further improve the sample complexity of the dictionary learning algorithm\footnote{In \cite{A16}, Adamczak analyze a slight modification of the dictionary learning algorithm proposed by Spielman et al.} to $n = O(m\log m)$~\cite{A16}. Since we rely on the dictionary learning algorithm in this paper, we describe the algorithm in Fig.~\ref{fig:ERSpUD} and present the exact recovery guarantees from \cite{A16}.

\begin{figure}[t!]
\algrule[1pt]
\textbf{Modiefied ER-SpUD (DC):} Exact recovery of sparsely-used dictionaries using the sum of two columns of $U$ as constraint vectors.
\algrule[1pt]
\begin{algorithmic}[1]
\REQUIRE $n$ observations $U = \left[\bfu_1, \bfu_2,\ldots, \bfu_n\right] \in \reals^{m \times n}$.
\STATE Initialize the set $\cV = \emptyset$.
\FOR{$i = 1,\ldots, n-1$}
\FOR{$j = i+1,\ldots, n$}
\STATE Let $\bfr_{ij} = \bfu_i + \bfu_j$. 
\STATE Solve ${\rm minimize}_{\bfw \in \RR^m} \|\bfw^TU\|_1$ subject to $\bfr_{ij}^T\bfw = 1$, and set $\bfs_{ij} = \bfw^TY \in \RR^n$.
\STATE $\cV = \cV \cup \{\bfs_{ij}\}$
\ENDFOR
\ENDFOR
\FOR{$i = 1,\ldots, m$}
\STATE \textbf{Repeat}
\STATE $\bfv_i \leftarrow {\rm argmin}_{\bfv \in \cV}\|\bfv\|_0$, breaking ties arbitrarily
\STATE $\cV = \cV \backslash \{\bfv_i\}$.
\STATE \textbf{Until} ${\rm rank}([\bfv_1, \bfv_2,\ldots, \bfv_i]) = i$.
\ENDFOR
\ENSURE $V = [\bfv_1, \bfv_2,\ldots, \bfv_m]^T$  and $D = UU^T(VV^T)^{-1}$.
\end{algorithmic}
\algrule[1pt]
\caption{Description of the dictionary learning algorithm from \cite{A16}.}
\label{fig:ERSpUD}
\end{figure}

\begin{proposition}
There exists absolute constants $c, \alpha \in (0, \infty)$ such that if 
$$
\frac{2}{m} \leq \theta \leq \frac{\alpha}{\sqrt{m}}
$$
and $V$ follows the Bernoulli-sub-Gaussian model with parameter $\theta$, then for $n \geq c m\log m$, with probability at least $1 - 1/n$ the modified ER-SpUD algorithm (cf.~Fig.~\ref{fig:ERSpUD}) successfully recovers all the rows of $V$, i.e., multiples of all the rows of $U$ are present among the set $\cV$. 
\end{proposition}

\subsection{Proof of Theorem~\ref{thm:learning}}
\label{appen:thm_learning}

In this section we highlight the proof of Theorem~\ref{thm:learning} which provides the guarantees for the exact recovery of the matrix $B$ using the learning algorithm described in Fig.~\ref{fig:ERSpUD}. In \cite[Theorem~1.1]{A16}, Adamczak establishes the analogue of Theorem~\ref{thm:learning} for $m \times n$ matrices generated by the Bernoulli-sub-Gaussian model (cf.~Sec.~\ref{sec:dictionary_background}). Theorem~\ref{thm:learning} can be established by suitably modifying the analysis of Adamczak which comprises four main steps as highlighted in \cite[Sec.~2.1]{A16}. Due to the small differences between the sparse-sub-Gaussian model (cf.~Sec.~\ref{sec:sys}) for the matrix $B$ considered in this paper and the Bernoulli-sub-Gaussian model from \cite{A16}, these steps continue to work after small modifications in the analysis. In the rest of this section, we demonstrate this by establishing Lemma~\ref{lem:modifie_lemma} for the sparse-sub-Gaussian model which is analogue to \cite[Lemma~2.4]{A16} for the Bernoulli-sub-Gaussian model. The analogue to other key lemmas from \cite{A16} can be similarly obtained.

Let's first define the required notation. In what follows, for $p \geq 1$, 
\begin{align}
\|\bfv\|_p := \Big(\sum_{i = 1}^mv_i^p\Big)^{1/p} \nonumber 
\end{align} 
denotes the $\ell_p$-norm of the vector $\bfv \in \RR^m$.  Moreover, we use $B_1^m \subset \RR^m$ to denote the set of $m$-length vectors with unit $\ell_1$-norm, i.e.,
\begin{align}
B_1^m := \big\{\bfv \in \RR^m~:~\|\bfv\|_1 = 1\big\}. \nonumber
\end{align}

In \cite{A16}, Adamczak proves the following concentration result using Bernstein's inequality and Talagrand's contraction principle. Here, we restate this result as it is utilized in the proof of Lemma~\ref{lem:modifie_lemma} below.

\begin{proposition}{\cite[Proposition 2.1]{A16}}
\label{prop:bern}
Let $R_1, R_2,\ldots, R_n \in \RR^m$ and $\xi_1, \xi_2,\ldots, \xi_n \in \{0, 1\}^m$ be two sets of independent random vectors. Assume that for some constant $L$, we have 
\begin{align}
\ep{e^{|R_{i, j}|/L}} \leq 2~~~~~\forall~1 \leq i \leq m,~1 \leq j \leq n.
\end{align}
Furthermore, assume that we have
\begin{align}
\pr{\xi_{i, j} = 1} \leq \theta~~~~~\forall~1 \leq i \leq m,~1 \leq j \leq n.
\end{align}
Let $Z_1, Z_2,\ldots, Z_n \in \RR^m$ be $n$ random vectors defined as follows.
\begin{align}
Z_j = \big(R_{1, j}\xi_{1,j},R_{2, j}\xi_{2,j},\ldots, R_{m, j}\xi_{m, j}\big)^T~~~~~\forall~1 \leq j \leq n.
\end{align}
Consider the random variable
\begin{align}
W = {\rm sup}_{\bfv \in B^{m}_1} \left|\frac{1}{n}\sum_{j = 1}^{n}\big(\bfv^TZ_j - \ep{\bfv^TZ_j}\big)\right|.
\end{align}
Then, for some universal constant $C$ and every $q \geq \max(2, \log m)$, we have
\begin{align}
\|W\|_q \leq \frac{C}{n}\big(\sqrt{nq\theta} + q\big)L
\end{align}
and 
\begin{align}
\pr{W \geq \frac{Ce}{n}\big(\sqrt{nq\theta} + q\big)L} \leq e^{-q}.
\end{align}
\end{proposition}

Before we proceed, we make the following simple claim about our generative model.
\begin{claim}
\label{clm:pnz}
For the random matrix ensemble generated by the sparse-sub-Gaussian model (cf.~Sec.~\ref{sec:sys}), whenever $d = o(m)$, we have have the following
\begin{align}
\big(1 - o(1)\big)\frac{d}{m}  \leq \pr{\xi_{i,j} = 1} = 1 - \left(1 - \frac{1}{m}\right)^d \leq \frac{d}{m}. \nonumber
\end{align}
\end{claim}

We now present the following result which is analogue to \cite[Lemma 2.4]{A16}. 

\begin{lemma}
\label{lem:modifie_lemma}
Let $\cS \subseteq [n]$ be a fixed subset of size $|\cS| < \frac{n}{4}$. Let $X \in \RR^{m \times n}$ be an $m \times n$ matrix which is generated as follows. 
\begin{itemize}
\item[(i)] For every $j \in \bar{\Sc} := [n]\backslash \cS$, we pick $d$ elements uniformly at random from $[m]$ with replacement. Let $\cN_j \subseteq [m]$ denote the set of picked elements. Let $R_{j} = (R_{1, j}, R_{2, j},\ldots, R_{m, j})^T \in \RR^m$ be a vector containing i.i.d. sub-Gaussian random variables and $\xi_{j} \in \{0, 1\}^m$ denote the indicator vector for the set $\cN_j \subseteq [m]$. Now the $j$-th column of the matrix $X$ is defined as $X_j = (\xi_{1, j}R_{1, j}, \xi_{2, j}R_{2, j},\ldots, \xi_{m, j}R_{m, j})^T \in \RR^m$. 
\item[(ii)] Let $s \leq 2d$. For every $j \in \cS$, we pick $d$ elements uniformly at random from the set $[m + s]$ with replacement. Let $\widetilde{\cN}_j \subseteq [m + s]$ denote the set of picked elements. We take a subset $\cN_j = \widetilde{\cN}_j \cap [m]$. Let $R_{j} = (R_{1, j}, R_{2, j},\ldots, R_{m, j})^T \in \RR^m$ be a vector containing i.i.d. sub-Gaussian random variables and $\xi_{j} \in \{0, 1\}^m$ denote the indicator vector for the set $\cN_j \subseteq [m]$. Now the $j$-th column of the matrix $X$ is defined as $X_j = (\xi_{1, j}R_{1, j}, \xi_{2, j}R_{2, j},\ldots, \xi_{m, j}R_{m, j})^T \in \RR^m$. 
\end{itemize}
Let $X_{\cS}$ denote the sub-matrix of $X$ comprising the columns indexed by the set $\cS \subseteq [n]$. Then, with probability at least $1 - n^{-8}$, for any $\bfv \in \reals^m$, we have 
\begin{align}
\|\bfv^TX\|_1 - 2\|\bfv^TX_{\cS}\|_1 \geq  \Omega\left({n \mu}\sqrt{\frac{\theta}{m}}\right)\|\bfv\|_1,
\end{align}
where $\theta = \frac{d}{m}$ and $\ep{|R_{i,j}|} \leq \mu$.
\end{lemma}
\begin{proof}
It follows from Proposition~\ref{prop:bern} that, with probability at least $1 - n^{-8}$, we have that
\begin{align}
\sup_{\bfv \in B_1^m} \left|\|\bfv^TX\|_1 - \ep{\|\bfv^TX\|}_1\right| &\leq C\left(\sqrt{n\theta\log n} + \log n\right) \nonumber\\
 & \leq 2C\sqrt{n\theta\log n} \nonumber
\end{align}
and
\begin{align}
\sup_{\bfv \in B_1^m} \left|\|\bfv^TX_{\cS}\|_1 - \ep{\|\bfv^TX_{\cS}\|}_1\right| &\leq 2C\sqrt{n\theta\log n}. \nonumber
\end{align}
This implies that for any $\bfv \in \RR^m$, we have that
\begin{align}
\|\bfv^TX\|_1 \geq  \ep{\|\bfv^TX\|}_1 - 2C\sqrt{n \theta\log n}\|\bfv\|_1 \nonumber
\end{align}
and
\begin{align}
\|\bfv^TX_{\cS}\|_1 \leq  \ep{\|\bfv^TX_{\cS}\|}_1  + 2C\sqrt{n \theta\log n}\|\bfv\|_1. \nonumber
\end{align}
Combining these two inequalities, we obtain that the following holds for each $\bfv \in \RR^m$.
\begin{align}
\label{eq:norm_diff1}
&\|\bfv^TX\|_1  - 2\|\bfv^TX_{\cS}\|_1 \geq \ep{\|\bfv^TX\|}_1 - 2\ep{\|\bfv^TX_{\cS}\|}_1 \nonumber \\
&~~~~~~~~~~~~~~~~~~~~~~~~~~~~~~~~~~~~~~~~~~~- 6C\sqrt{n\theta\log n}\|\bfv\|_1 \nonumber \\
&= \sum_{j \in \cS}\ep{\left|\bfv^TX_j\right|} + \sum_{j \in \bar{\cS}}\ep{\left|\bfv^TX_j\right|} - 2 \sum_{j \in \cS}\ep{\left|\bfv^TX_j\right|} \nonumber \\
&~~~~~~~~- 6C\sqrt{n\theta\log n}\|\bfv\|_1 \nonumber \\
&= \sum_{j \in \bar{\cS}}\ep{\left|\bfv^TX_j\right|} - \sum_{j \in \cS}\ep{\left|\bfv^TX_j\right|} \nonumber \\
&~~~~~~~~-  6C\sqrt{n\theta\log n}\|\bfv\|_1. 
\end{align}
We now bound $\ep{\left|\bfv^TX_j\right|}$ for $j \in \cS$. Recall that all the columns of the matrix $X$ indexed by the set $\cS$ are identically distributed. Similarly, all the columns of the matrix $X$ indexed by the set $\bar{\cS} = [n]\backslash \Sc$ are identically distributed. In the following, we use $Z = (Z_1,\ldots, Z_m)^T$ and $\widehat{Z}= (\widehat{Z}_1,\ldots, \widehat{Z}_m)^T$ to denote two random vectors with their distribution identical to the columns of the matrix $X$ indexed by the set $\bar{\cS}$ and ${\cS}$, respectively.
\begin{align}
\label{eq:norm_diff2}
\ep{\big|\bfv^T\widehat{Z}\big|} &=  \ep{\big|\bfv^T(\widehat{Z} + Z - Z)\big|} \nonumber \\
&\leq \ep{\big|\bfv^T{Z}\big|} + \ep{\big|\bfv^T(Z - \widehat{Z})\big|}  \nonumber \\
&\leq \ep{\big|\bfv^T{Z}\big|}  + \mu\ep{\sum_{i = 1}^m|v_i|W_i}.
\end{align}
Here, for $1 \leq i \leq m$, $W_i = \sum_{l = 1}^{d}Y^l_{i}$ denotes the sum of $d$ indicator random variable which are defined as follows. 
\begin{align}
\label{eq:norm_diff3}
Y_{i}^{l} = \begin{cases}
1 &\mbox{with probability}~\frac{2d}{m + 2d}\frac{1}{m} \\
0 &\mbox{with probability}~1 - \frac{2d}{m + 2d}\frac{1}{m}.
\end{cases}
\end{align}
Combining \eqref{eq:norm_diff2} and \eqref{eq:norm_diff3}, we obtain
\begin{align}
\label{eq:norm_diff4}
\ep{\big|\bfv^T\widehat{Z}\big|}  &\leq \ep{\big|\bfv^T{Z}\big|}  +  \mu\frac{2d^2}{(m + 2d)m}\|\bfv\|_1 \nonumber \\
&\leq \ep{\big|\bfv^T{Z}\big|}  +  \mu\frac{2d^2}{m^2}\|\bfv\|_1.
\end{align}
Note that we have 
\begin{align}
\ep{\big|\bfv^TX_j\big|} &= \ep{\big|\bfv^T\widehat{Z}\big|}~~~~~\forall~j \in \cS \\
\ep{\big|\bfv^TX_j\big|} &= \ep{\big|\bfv^T{Z}\big|}~~~~~\forall~j \in \bar{\cS} = [n]\backslash\cS.
\end{align}
Therefore, combining \eqref{eq:norm_diff1} and \eqref{eq:norm_diff4}, we obtain that for any $\bfv \in \RR^m$,
\begin{align}
\label{eq:norm_diff5}
&\|\bfv^TX\|_1  - 2\|\bfv^TX_{\cS}\|_1 \geq (p - 2|\cS|)\ep{\big|\bfv^T{Z}\big|} \nonumber \\ 
&~~~~~~~~~~~~~~~~~~~- |\cS|\mu\frac{2d^2}{m^2}\|\bfv\|_1 - 4C\sqrt{n\theta\log n}\|\bfv\|_1.
\end{align} 
Now using $m = c\frac{n}{\log n}$, $d \leq c'' \log n$ and the lower bound on $\ep{\big|\bfv^T{Z}\big|}$ from \cite[Lemma 2.3]{A16}, one can argue that 
\begin{align}
\label{eq:norm_diff5}
\|\bfv^TX\|_1  - 2\|\bfv^TX_{\cS}\|_1 & \geq \Omega\left({n \mu}\sqrt{\frac{\theta}{m}}\right)\|\bfv\|_1.
\end{align} 
\end{proof}

\end{document}